\documentclass{article}
\usepackage[letterpaper, left=1in, right=1in, top=1in, bottom=1in]{geometry}
\usepackage{natbib}
\bibpunct{(}{)}{;}{a}{,}{,}
\usepackage[colorlinks,citecolor=blue!70!black,linkcolor=blue!70!black,
urlcolor=blue!70!black,breaklinks=true]{hyperref}
\usepackage{parskip}
\usepackage[usenames, dvipsnames]{xcolor}
\usepackage{microtype}
\usepackage{booktabs}
\usepackage{amsmath}
\usepackage{dsfont}
\usepackage{amssymb}
\usepackage{xspace}
\usepackage{url}
\usepackage{algorithm}
\usepackage[noend]{algorithmic}
\usepackage{amsthm}
\usepackage{listings}
\usepackage{bm}
\usepackage{bbm}
\usepackage{enumitem}
\usepackage{nicefrac}
\usepackage{mathrsfs}
\usepackage{inconsolata}
\usepackage[subrefformat=parens,labelformat=parens]{subfig}
\usepackage{wrapfig}
\usepackage{graphicx}
\usepackage{bigints}
\usepackage{transparent}
\usepackage{comment}
\usepackage{float}
\usepackage[nameinlink,capitalize]{cleveref}
\makeatletter
\AddToHook{cmd/appendix/before}{
\def\cref@section@alias{appendix}
\def\cref@subsection@alias{appendix}
\def\cref@subsubsection@alias{appendix}
}
\makeatother
\usepackage{thmtools}
\usepackage{thm-restate}
\usepackage{nicematrix}
\usepackage{arydshln}
\usepackage{fancyhdr}
\usepackage{mathtools}
\usepackage[scaled=.88]{helvet}
\usepackage{breakcites}
\usepackage{multirow}
\usepackage[section,below]{placeins}

\graphicspath{{./figs/}}

\def\E{{\mathbb E}}

\renewcommand{\epsilon}{\varepsilon}

\newcommand{\indic}{\mathbb{I}}

\newtheoremstyle{spaced}
  {6pt}   %
  {0pt}   %
  {\itshape} %
  {}       %
  {\bfseries} %
  {.}      %
  {0.5em}  %
  {}
\theoremstyle{spaced}

\DeclarePairedDelimiter{\brk}{[}{]}
\DeclarePairedDelimiter{\crl}{\{}{\}}

\DeclarePairedDelimiterX{\infdiv}[2]{(}{)}{%
  #1\;\delimsize\|\;#2%
}

\newcommand{\wh}[1]{\widehat{#1}}

\def\ddefloop#1{\ifx\ddefloop#1\else\ddef{#1}\expandafter\ddefloop\fi}
\def\ddef#1{\expandafter\def\csname bb#1\endcsname{\ensuremath{\mathbb{#1}}}}
\ddefloop ABCDEFGHIJKLMNOPQRSTUVWXYZ\ddefloop
\def\ddefloop#1{\ifx\ddefloop#1\else\ddef{#1}\expandafter\ddefloop\fi}
\def\ddef#1{\expandafter\def\csname b#1\endcsname{\ensuremath{\mathbf{#1}}}}
\ddefloop ABCDEFGHIJKLMNOPQRSTUVWXYZ\ddefloop
\def\ddef#1{\expandafter\def\csname sf#1\endcsname{\ensuremath{\mathsf{#1}}}}
\ddefloop ABCDEFGHIJKLMNOPQRSTUVWXYZ\ddefloop
\def\ddef#1{\expandafter\def\csname c#1\endcsname{\ensuremath{\mathcal{#1}}}}
\ddefloop ABCDEFGHIJKLMNOPQRSTUVWXYZ\ddefloop
\def\ddef#1{\expandafter\def\csname h#1\endcsname{\ensuremath{\widehat{#1}}}}
\ddefloop ABCDEFGHIJKLMNOPQRSTUVWXYZ\ddefloop
\def\ddef#1{\expandafter\def\csname hc#1\endcsname{\ensuremath{\widehat{\mathcal{#1}}}}}
\ddefloop ABCDEFGHIJKLMNOPQRSTUVWXYZ\ddefloop
\def\ddef#1{\expandafter\def\csname t#1\endcsname{\ensuremath{\widetilde{#1}}}}
\ddefloop ABCDEFGHIJKLMNOPQRSTUVWXYZ\ddefloop
\def\ddef#1{\expandafter\def\csname tc#1\endcsname{\ensuremath{\widetilde{\mathcal{#1}}}}}
\ddefloop ABCDEFGHIJKLMNOPQRSTUVWXYZ\ddefloop
\def\ddefloop#1{\ifx\ddefloop#1\else\ddef{#1}\expandafter\ddefloop\fi}
\def\ddef#1{\expandafter\def\csname scr#1\endcsname{\ensuremath{\mathscr{#1}}}}
\ddefloop ABCDEFGHIJKLMNOPQRSTUVWXYZ\ddefloop

\let\oldparagraph\paragraph
\renewcommand{\paragraph}[1]{\oldparagraph{#1}}

\renewcommand{\epsilon}{\varepsilon}

\newcommand{\SE}{\mathrm{SE}}
\newcommand{\SC}{\mathrm{SC}}

\DeclarePairedDelimiter{\paren}{(}{)}

\renewcommand{\bigm}[1]{%
  \ifcsname fenced@\string#1\endcsname
    \expandafter\@firstoftwo
  \else
    \expandafter\@secondoftwo
  \fi
  {\expandafter\amsmath@bigm\csname fenced@\string#1\endcsname}%
  {\amsmath@bigm#1}%
}

\newcommand{\DeclareFence}[2]{\@namedef{fenced@\string#1}{#2}}
\makeatother

\makeatletter
\let\save@mathaccent\mathaccent
\newcommand*\if@single[3]{%
  \setbox0\hbox{${\mathaccent"0362{#1}}^H$}%
  \setbox2\hbox{${\mathaccent"0362{\kern0pt#1}}^H$}%
  \ifdim\ht0=\ht2 #3\else #2\fi
  }
\newcommand*\rel@kern[1]{\kern#1\dimexpr\macc@kerna}
\newcommand*\widebar[1]{\@ifnextchar^{{\wide@bar{#1}{0}}}{\wide@bar{#1}{1}}}
\newcommand*\wide@bar[2]{\if@single{#1}{\wide@bar@{#1}{#2}{1}}{\wide@bar@{#1}{#2}{2}}}
\newcommand*\wide@bar@[3]{%
  \begingroup
  \def\mathaccent##1##2{%
    \let\mathaccent\save@mathaccent
    \if#32 \let\macc@nucleus\first@char \fi
    \setbox\z@\hbox{$\macc@style{\macc@nucleus}_{}$}%
    \setbox\tw@\hbox{$\macc@style{\macc@nucleus}{}_{}$}%
    \dimen@\wd\tw@
    \advance\dimen@-\wd\z@
    \divide\dimen@ 3
    \@tempdima\wd\tw@
    \advance\@tempdima-\scriptspace
    \divide\@tempdima 10
    \advance\dimen@-\@tempdima
    \ifdim\dimen@>\z@ \dimen@0pt\fi
    \rel@kern{0.6}\kern-\dimen@
    \if#31
      \overline{\rel@kern{-0.6}\kern\dimen@\macc@nucleus\rel@kern{0.4}\kern\dimen@}%
      \advance\dimen@0.4\dimexpr\macc@kerna
      \let\final@kern#2%
      \ifdim\dimen@<\z@ \let\final@kern1\fi
      \if\final@kern1 \kern-\dimen@\fi
    \else
      \overline{\rel@kern{-0.6}\kern\dimen@#1}%
    \fi
  }%
  \macc@depth\@ne
  \let\math@bgroup\@empty \let\math@egroup\macc@set@skewchar
  \mathsurround\z@ \frozen@everymath{\mathgroup\macc@group\relax}%
  \macc@set@skewchar\relax
  \let\mathaccentV\macc@nested@a
  \if#31
    \macc@nested@a\relax111{#1}%
  \else
    \def\gobble@till@marker##1\endmarker{}%
    \futurelet\first@char\gobble@till@marker#1\endmarker
    \ifcat\noexpand\first@char A\else
      \def\first@char{}%
    \fi
    \macc@nested@a\relax111{\first@char}%
  \fi
  \endgroup
}
\makeatother

\newcommand{\algname}{Active Testing via Approximate Neyman Allocation\xspace}

\title{Active Testing of Large Language Models via Approximate Neyman Allocation}
\date{}

\author{
Zeli Liu\\
{University of California, Riverside}\\
{\texttt{zliu466@ucr.edu}}
\and
Jiancheng Zhang\\
{University of California, Riverside}\\
{\texttt{jzhan745@ucr.edu}}
\and
Cong Liu\\
{University of California, Riverside}\\
{\texttt{congl@ucr.edu}}
\and
Yinglun Zhu\\
{University of California, Riverside}\\
{\texttt{yzhu@ucr.edu}}
}

\renewcommand{\thefootnote}{\fnsymbol{footnote}}
\begin{document}

\maketitle
\setcounter{footnote}{0}
\renewcommand{\thefootnote}{\arabic{footnote}}

\begin{abstract}
  Large language models (LLMs) require reliable evaluation from pre-training to test-time scaling, making evaluation a recurring rather than one-off cost. As model scales grow and target tasks increasingly demand expert annotators, both the compute and labeling costs needed for each evaluation rise rapidly. Active testing aims to alleviate this bottleneck by approximating the evaluation result from a small but informative subset of the evaluation pool. However, existing approaches primarily target classification and break down on generative tasks. We introduce a novel active testing algorithm tailored to generative tasks. Our method leverages semantic entropy from surrogate models to stratify the evaluation pool and then conducts approximate Neyman allocation based on signals extracted from these surrogates. Across multiple language and multimodal benchmarks and a range of surrogate-target model pairs, our method significantly improves on baselines and closely tracks Oracle-Neyman, delivering up to 28\% MSE reduction over Uniform Sampling and an average of 22.9\% budget savings.

\end{abstract}

\section{Introduction} \label{sec:intro}

Large language models (LLMs) are increasingly deployed across diverse reasoning, multimodal, and agentic tasks \citep{wei2022chain,liu2023visual,zhu2025test}, often via API-based services. Reliable evaluation of these models is needed repeatedly throughout model development, including model selection during pre- and post-training, regression testing across successive releases, safety and quality monitoring after deployment, and coverage across the long tail of domains and sub-tasks each model is asked to handle. The per-evaluation cost is also growing rapidly along two axes. On the inference side, four layers of scaling---pre-training, post-training, test-time, and agentic paradigms \citep{hoffmann2022training,ouyang2022training,snell2024scaling,yao2022react,zuo2025strategic, zuo2026adaptive}---substantially increase the compute spent on every input. On the annotation side, many target tasks extend beyond the capabilities of standard annotators, since, in the absence of established benchmarks for niche or real-world inputs, grading frequently requires costly domain experts.

Active testing has been proposed as a principled approach to this regime, selecting a small but informative subset for labeling and using an estimator to recover full-pool performance \citep{kossen2021active, farquhar2021statistical, berrada2025scaling}. Existing pipelines, however, fall short of the requirements of modern generative LLM evaluation: they are designed for classification rather than for multi-token generation, often require retraining a data selector, or depend on internal activations and logits that black-box API services generally do not provide. We analyze the underlying failure modes in \cref{sec:method-motivation}.

\begin{figure}[t]
\centering
\includegraphics[width=0.99\linewidth]{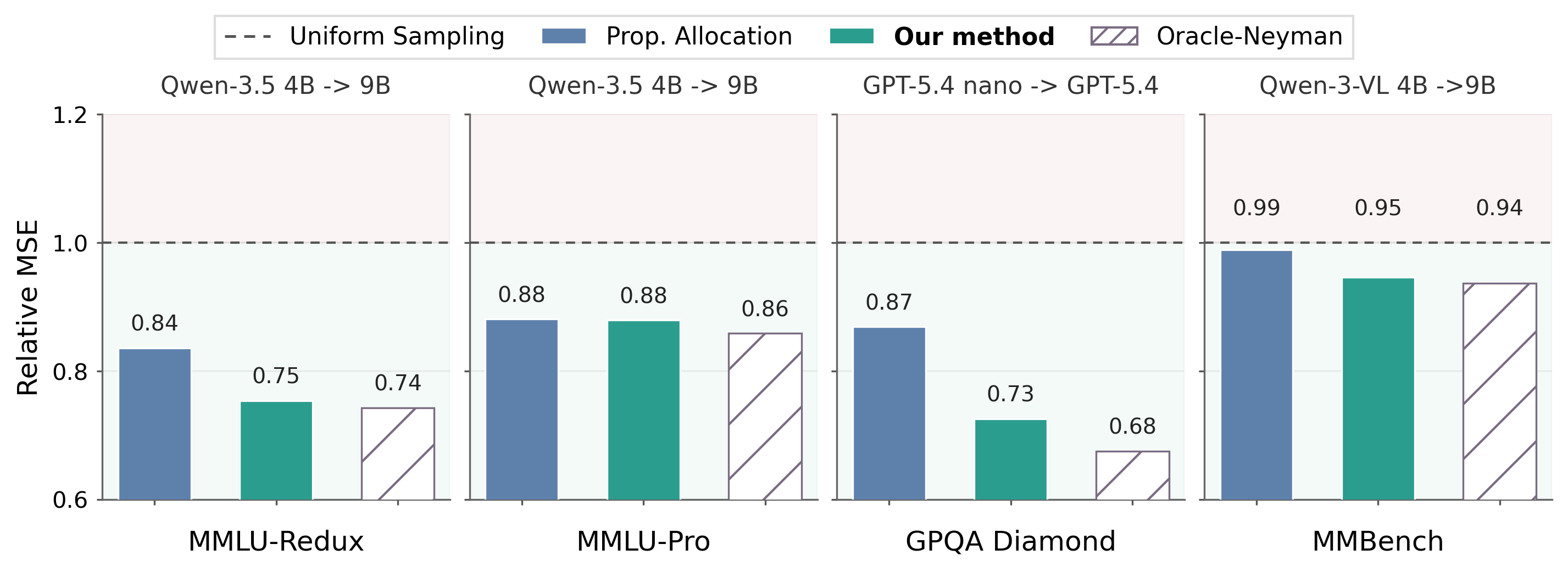}
\caption{Performance comparison across four benchmarks with labeling budget $M$=70 (\emph{MSE; lower is better}). Our method outperforms the Uniform Sampling baseline and closely approaches the Oracle-Neyman (theoretical optimum), yielding consistent gains across different modalities, multiple benchmarks, and both representative Qwen and state-of-the-art GPT models.}
\label{fig:intro_budget70}
\end{figure}

In this paper, we extend active testing to generative tasks and propose a training-free and theoretically unbiased active-testing method applicable to both large language and multimodal models. Our central insight is that uncertainty in generative outputs is fundamentally a property of meaning rather than tokens, and therefore cannot be adequately captured by classification-style entropy measures. Our method leverages \emph{semantic entropy} \citep{kuhn2023semantic, farquhar2024detecting} as a meaning-level uncertainty signal that reflects generative ambiguity in model outputs. Building on this signal, our method constructs a stratified estimator equipped with a surrogate-driven Neyman-style allocation \citep{neyman1992two}, while remaining fully training-free and requiring neither selector optimization nor access to internal activations or logits. As shown in \cref{fig:intro_budget70}, our method closely tracks Oracle-Neyman across several language and multimodal benchmarks.

\paragraph{Our contributions.}
Our main contributions are as follows:
\begin{enumerate}[label=(\roman*)]
\item To the best of our knowledge, our method is the first active-testing pipeline that provides a theoretically unbiased, label-efficient estimator for generative language and multimodal evaluation while simultaneously supporting training-free operation and API-only deployment.
\item We develop a new algorithm that leverages surrogate-driven semantic entropy to structure the sampling process, couples this signal with a Neyman-style allocation strategy, and yields a training-free, unbiased estimator for evaluating both language and multimodal models.
\item We conduct extensive experiments across multiple language and multimodal benchmarks with both open-source and advanced proprietary models. Our method delivers consistent variance reduction over Uniform Sampling across labeling budgets, reduces MSE by up to $28\%$ in the representative $M=70$ setting in \cref{fig:intro_budget70}, and achieves an average of $22.9\%$ budget savings on expert annotations at matched precision.

\end{enumerate}

\paragraph{Paper organization.} The remainder of the paper is organized as follows. \cref{sec:related-work} reviews related work. \cref{sec:setting} introduces the problem formulation. \cref{sec:methods} presents our method along with its theoretical analysis. \cref{sec:experiments} reports the main experimental results and additional analyses. \cref{sec:discussion} concludes the paper. Proofs, implementation details, and additional experimental results are deferred to the Appendix.

\section{Related Work} \label{sec:related-work}

\paragraph{Active testing.} Active testing was introduced by \citet{kossen2021active} for label-efficient model evaluation, using a surrogate-driven acquisition function debiased by the LURE estimator \citep{kossen2022active, farquhar2021statistical}, and later extended to LLMs by \citet{berrada2025scaling} via a frozen in-context surrogate that removes the retraining loop. Both methods are tied to the classification setting, which mismatches the generative deployment of modern LLMs.

AcTracer \citep{huang2026actracer} performs multi-stage sampling which relies on hidden states unavailable in black-box APIs, and its confidence-matching proxy does not yield a theoretically unbiased estimator. GAT \citep{ramakrishnan2026generative} retains the LURE estimator but rewrites each input into a True/False proxy task with first-token acquisition, again mismatched with multi-token generative evaluation. \citet{zhou2025accelerating} study efficient pairwise win-rate estimation between two LLMs via control variates, whereas we target single-model evaluation.

\paragraph{Sampling, stratification, and budget allocation.} Sampling theory provides classical tools for estimating population quantities from a small queried subset \citep{cochran1977sampling, lohr2021sampling}. Variance-reducing designs include stratified sampling \citep{neyman1992two, dalenius1959minimum}, unequal-probability sampling \citep{hansen1943theory}, systematic sampling \citep{madow1944theory}, cluster and multi-stage sampling, and adaptive designs \citep{thompson1990adaptive}. The companion question of budget allocation across strata determines how labels are distributed, ranging from equal allocation and proportional allocation to variance-optimal Neyman allocation under known stratum standard deviations \citep{neyman1992two}. Such designs admit unbiased finite-population estimation through inverse-probability weighting \citep{horvitz1952generalization}.

\paragraph{Active learning.}
Efficient utilization of labeling resources is also a central objective in model training, alongside broader efficiency considerations \citep{wang2024moq,wang2025fedpai,li2025mixtraining}.
Active learning aims to select the most informative subset from an unlabeled pool for annotation, thereby reducing labeling costs \citep{settles2009active}. Traditional active learning methods are typically based on either diversity \citep{sener2017active,agarwal2020contextual} or uncertainty \citep{tong2001support,scheffer2001active}. It has been shown to improve data efficiency from both theoretical perspectives \citep{krishnamurthy2019active,puchkin2021exponential,zhu2022active,zhu2022efficient} and empirical studies \citep{saran2023streaming,zhang2023labelbench}. Recently, active learning has also emerged as a common strategy for improving data efficiency in large pretrained models \citep{bhatt2024experimental,yuan2024hide,zhang2025towards}.

\section{Problem Setting}
\label{sec:setting}

\paragraph{Active testing.} We study active testing for large language and multimodal models.
We denote $f$ as the target model and let $\cD = \crl{x_i}_{i=1}^N$ be a finite pool of unlabeled test inputs drawn from an evaluation distribution $p_{\mathrm{eval}}$.
We denote $\ell$ as the loss function of predictive errors. Following \citet{berrada2025scaling}, the ideal population risk of interest is:
\begin{equation*}
R
\;=\;
\E_{(x,y)\sim p_{\mathrm{eval}}}\!\brk*{\ell(f(x), y)}.
\end{equation*}
Since $p_{\mathrm{eval}}$ and the labels of all possible test inputs are unknown, we condition on the fixed pool $\cD$ and instead estimate the finite-pool risk:
\begin{equation*}
R_{\mathrm{D}}
\;=\;
\frac{1}{N} \sum_{i=1}^{N} \ell \big(f(x_i), y_i\big).
\end{equation*}

Active testing avoids labeling the full pool by selecting a subset $\cS \subset \cD$ of size $M \ll N$ and constructing an estimator $\wh{R}$ for $R_{\mathrm{D}}$ from the labeled subset \citep{kossen2021active}.
The objective is to minimize the estimation error of $\wh{R}$ under a fixed labeling budget while preserving unbiasedness with respect to the full-pool risk, $\E[\wh{R}] = R_{\mathrm{D}}$.

\paragraph{Evaluation metrics.} Following \citet{berrada2025scaling}, we evaluate a testing method by the squared error of its risk estimate against the ground-truth finite-pool risk $R_{\mathrm{D}}$, where $R_{\mathrm{D}}$ is obtained by labeling the entire pool. For a fixed labeling budget $M$, let $\wh{R}^{(t)}$ denote the estimate produced by estimator $\wh{R}$ in Monte Carlo trial $t$, and let $T=3{,}000$ be the number of trials. We compute the mean squared error (MSE)
\begin{equation*}
\mathrm{MSE}(\wh{R}) = \frac{1}{T}\sum_{t=1}^{T}\paren*{\wh{R}^{(t)} - R_{\mathrm{D}}}^2.
\end{equation*}

To enable comparisons across different benchmarks, we also follow \citet{berrada2025scaling} and report the error relative to the Uniform Sampling baseline. Let $\wh{R}_{\mathrm{unif}}$ denote the estimator induced by {Uniform Sampling}. We report the relative-MSE ratio $r(\wh{R}) = \frac{\mathrm{MSE}(\wh{R})}{\mathrm{MSE}(\wh{R}_{\mathrm{unif}})}$.

\section{Method} \label{sec:methods}

\begin{figure}[!t]
\centering
\includegraphics[width=\linewidth]{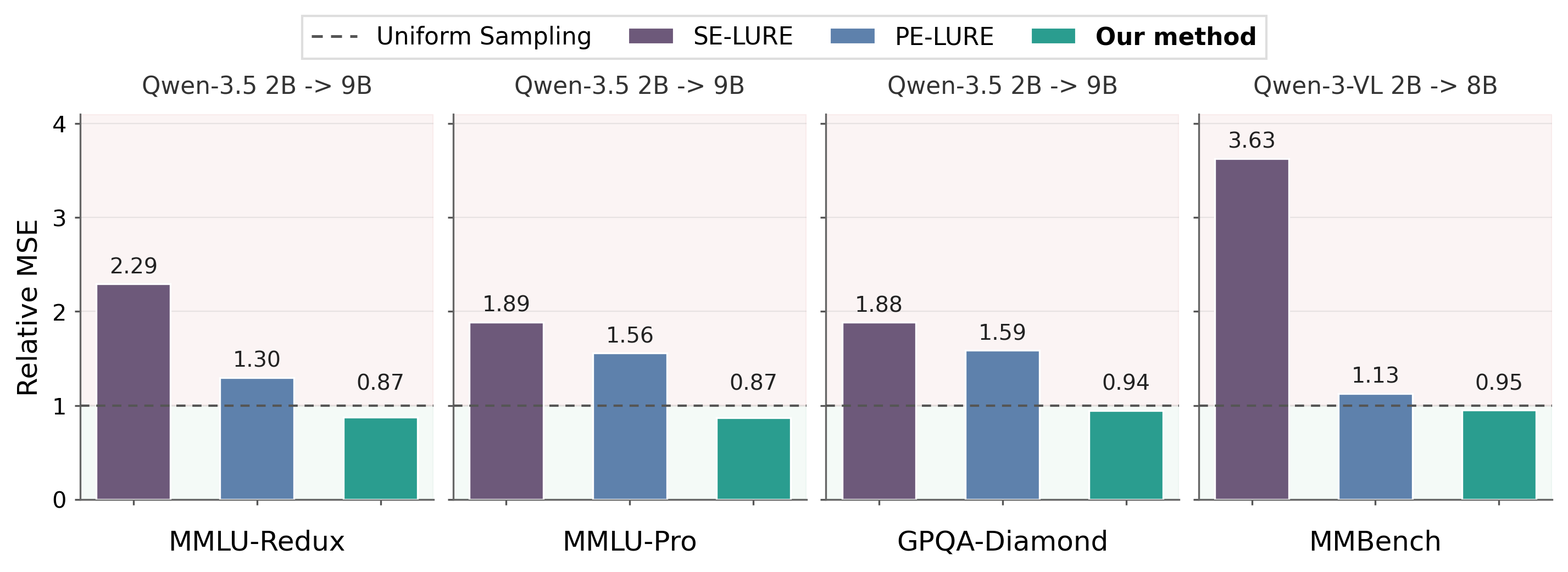}
\caption{\textbf{Comparison with adapted LURE variants.} We report the relative MSE (lower is better) of SE-LURE, PE-LURE, and our method against Uniform Sampling across four benchmarks and model pairs at budget $M = 100$.}

\label{fig:lure_mse}
\end{figure}

In this section, we first present the motivation of our approach in \cref{sec:method-motivation}. Building on this, we introduce the proposed method and provide its unbiasedness guarantee in \cref{sec:method-guarantees}.

\subsection{Motivation} \label{sec:method-motivation}

A natural way to extend the method of \citet{berrada2025scaling}, which provides an effective solution for active testing in classification tasks, to generative settings is to retain their LURE estimator and replace first-token predictive entropy (PE) with token-averaged PE computed over multi-token generations from a surrogate model. However, this straightforward PE-based adaptation does not outperform Uniform Sampling, as shown in \cref{fig:lure_mse}.
We identify two factors underlying this failure. \textit{First, token-level PE entangles semantic uncertainty with surface-form variation}: semantically equivalent generations with different phrasings induce different token-level logit sequences, causing PE to be inflated by paraphrastic diversity rather than genuine difficulty \citep{kuhn2023semantic,farquhar2024detecting}. \textit{Second, modern LLMs are systematically overconfident}, collapsing PE into a narrow low-entropy band and weakening the acquisition signal; surrogate and target PE remain only moderately correlated in our analysis, yielding a substantial mass of instances with systematic disagreement, where surrogate and target confidence sharply diverge \citep{farquhar2024detecting}.

Based on the analysis of the failure factors, we further consider replacing PE with semantic entropy (SE), a meaning-level uncertainty signal that we introduce in \cref{sec:method-pipeline}. However, this SE-LURE estimator also fails across various datasets, as shown in \cref{fig:lure_mse}. This suggests that the failure is not due to the uncertainty signal, but to the LURE estimator itself. In particular, LURE applies per-sample importance weights ($v_m \propto 1/a_m(i_m)$) \citep{berrada2025scaling}, which amplify surrogate--target mismatches. When the surrogate assigns low acquisition probability to instances that the target is actually uncertain about, these weights dominate the estimator. Since such mismatches are systematic rather than rare, improving the entropy signal alone does not reduce variance but instead exposes this instability.

This observation indicates that the core issue lies in the importance-reweighting mechanism, suggesting the need for an alternative estimation strategy that avoids such variance amplification. Motivated by this, we turn to a Neyman-style stratified estimator, which allocates the labeling budget in a variance-aware but unbiased manner.

\subsection{Our Method} \label{sec:method-pipeline}

\begin{algorithm}[t]
    \renewcommand{\algorithmicrequire}{\textbf{Input:}}
    \renewcommand{\algorithmicensure}{\textbf{Output:}}
    \caption{\algname}
    \label{alg:method}
    \begin{algorithmic}[1]
      \setlength{\itemsep}{2pt}

      \REQUIRE Pool set $\cD$, target model $f$, surrogate model $g$, label budget $M$.
      \ENSURE Estimate of full-pool performance $\wh{R}$.

    \STATE \textbf{Signal generation.}\ \ Compute semantic entropy $\SE(g(x_i))$ and calculate self-consistency score $\SC(g(x_i))$ for all $x_i \in \cD$, respectively: 
    \begin{equation}   \label{eq:se-kuhn}
     \SE(g(x_i))
      \;=\;
      -\sum_{c \in \cC_i}\paren*{\sum_{s \in c} p(s \mid g(x_i))}
      \log\paren*{\sum_{s \in c} p(s \mid g(x_i))},      
    \end{equation}
    
    \begin{equation}\label{eq:sc}
     \SC(g(x_i)) \;=\; \frac{1}{k} \max_{c \in \cC_i} \sum_{j=1}^k \indic[g(x_i)^{(j)} = c].
    \end{equation}
    \STATE \textbf{Stratification on semantic entropy.}\ \ Partition $\cD$ into $H$ strata $\cD_0, \ldots, \cD_{H-1}$ by binning instances according to their semantic entropy $\SE(g(x_i))$, set $N_h = |\cD_h|$, and compute the mean self-consistency score in each stratum:
    \begin{equation*}
      p_h \;=\; \frac{1}{N_h}\sum_{x_i \in \cD_h} \SC(g(x_i)).
    \end{equation*}
    \STATE \textbf{Budget allocation.}\ \  Allocate $m_h$ labels per stratum:
    \begin{equation}
m_h \;\propto\; N_h \cdot \paren*{\sqrt{p_h(1-p_h)} +  \delta},
\qquad \sum_h m_h = M, \qquad 1 \leq m_h \leq N_h.
\label{eq:allocation}
\end{equation}

    \STATE \textbf{Estimation.}\ \ Sample $\cS_h \subset \cD_h$ uniformly without replacement with $|\cS_h| = m_h$, and estimate full-pool performance:
    
\begin{equation}
\wh{R}
\;=\;
\frac{1}{N} \sum_{h=0}^{H-1} N_h \,\bar{\ell}_h,
\qquad
\bar{\ell}_h \;=\; \frac{1}{m_h}\sum_{x_i \in \cS_h}\ell(f(x_i), y_i).
\label{eq:estimator}
\end{equation}

    \end{algorithmic}
\end{algorithm}

We present our method in \cref{alg:method}, a training-free and theoretically unbiased active testing framework.
The proposed estimator constructs SE-based strata from surrogate-derived uncertainty signals, performs a proxy Neyman-style allocation of the labeling budget across strata, and aggregates the per-stratum target losses with a Horvitz--Thompson estimator. We next explain the details of each of the four steps in \cref{alg:method}.

\paragraph{Signal generation.} This step generates uncertainty signals by running the surrogate model $g$ on the evaluation pool $\cD$ to produce $k$ generations per input ($k = 10$ throughout). The generations are used to calculate semantic entropy and self-consistency scores. This surrogate pass incurs only modest overhead, as $g$ is substantially smaller than the target model $f$, and the $k$ generations share a single prefill through KV-cache reuse supported by modern inference frameworks \citep{kwon2023efficient}.

Semantic entropy addresses the fact that token-level predictive entropy in generative models can conflate uncertainty over \emph{meaning} with uncertainty over surface form \citep{kuhn2023semantic, farquhar2024detecting}. For each input $x_i$, it clusters generations into meaning-equivalence classes $\cC_i$ and computes entropy over the induced class probabilities, as in \cref{eq:se-kuhn}. When the answer space is closed, as in our benchmarks after parsing, the equivalence classes coincide with distinct parsed answers and \cref{eq:se-kuhn} reduces to the Shannon entropy over parsed-answer frequencies. 

Based on the generated samples, we construct two training-free uncertainty signals: the per-sample semantic entropy $\SE(g(x_i))$ and the self-consistency score \citep{wang2022self}, defined as in \cref{eq:sc}. Here $g(x_i)^{(j)}$ is the $j$-th surrogate generation for input $x_i$. Per-sample semantic entropy captures the surrogate's uncertainty on sample $x_i$, while the self-consistency score measures agreement among surrogate generations (higher values indicate stronger consensus and easier instances). It is used to construct the Bernoulli proxy for Neyman allocation described below.

\paragraph{Stratification on semantic entropy.} We partition $\cD$ into $H$ strata $\cD_0, \ldots, \cD_{H-1}$ of sizes $N_h = |\cD_h|$ based on semantic entropy. In practice, a non-trivial fraction of samples satisfy $\SE(g(x_i))=0$, so we assign them to a base stratum $\cD_0 = \{x_i \in \cD : \SE(g(x_i))=0\}$ and apply equal-quantile binning to the remaining samples $\{x_i \in \cD : \SE(g(x_i)) > 0\}$. We set $H=5$ throughout and provide an ablation over $H$ in \cref{sec:exp-ablations}.

\paragraph{Budget allocation.} Given a fixed stratification, classical Neyman allocation $m_h^* \propto N_h \sigma_h$ provides a theoretical optimum that minimizes the estimation variance \citep{neyman1992two}. However, this optimum depends on the within-stratum standard deviation $\sigma_h$ of the target loss, which is not directly accessible without labeling the entire pool. To address this, we replace $\sigma_h$ with a smoothed Bernoulli proxy $\sqrt{p_h(1-p_h)} + \delta$, where $p_h = \frac{1}{N_h}\sum_{i \in \cD_h} \SC(g(x_i))$ denotes the mean surrogate self-consistency within stratum $\cD_h$. The resulting allocation is defined as in \cref{eq:allocation}. The offset $\delta > 0$ prevents $m_h$ from collapsing when $p_h$ approaches~$1$ in the base stratum. We use $\delta = 0.75$ and report an ablation study in \cref{sec:exp-ablations}.

\paragraph{Estimation.}
Within each stratum $\cD_h$, we draw $m_h$ samples uniformly without replacement and obtain their labels. Let $\cS_h \subset \cD_h$ denote the selected subset, and define the within-stratum empirical loss mean
\begin{equation*}
\bar{\ell}_h \;=\; \frac{1}{m_h} \sum_{i \in \cS_h} \ell(f(x_i), y_i).
\end{equation*}
Since each sample in stratum $\cD_h$ is selected with equal inclusion probability $\pi_h = m_h / N_h$, we can write the Horvitz--Thompson form of $\wh{R}$ over the globally selected set $\cS = \bigcup_{h=0}^{H-1} \cS_h$:
\begin{equation*}
\wh{R} \;=\; \frac{1}{N} \sum_{i \in \cS} \frac{\ell(f(x_i), y_i)}{\pi_{i}}.
\end{equation*}
Under stratified simple random sampling, where $\pi_i = \pi_h$ for all $x_i \in \cD_h$, this expression simplifies to \cref{eq:estimator}, which is a theoretically unbiased estimator for the finite-pool risk $R_{\mathrm{D}}$, as we discuss next.

\subsection{Theoretical Analysis} \label{sec:method-guarantees}

We establish the following theorem to formalize the theoretical unbiasedness of our method.

\begin{restatable}{theorem}{Unbiasedness}
\label{thm:unbiasedness}
Our method yields an unbiased estimator of the finite-pool risk:
\[
\mathbb{E}\brk*{\wh{R}} \;=\; R_{\mathrm{D}},
\]
for any stratification $\{\cD_h\}_{h=0}^{H-1}$ of $\cD$ and any allocation $\{m_h\}_{h=0}^{H-1}$ satisfying $1 \leq m_h \leq N_h$ and $\sum_{h=0}^{H-1} m_h = M$. The result follows from (i) unbiasedness of within-stratum uniform sampling without replacement, and (ii) linearity of expectation for the stratified Horvitz--Thompson estimator.
\end{restatable}

We defer the formal proof to \cref{app:proof}.

\section{Empirical Results} \label{sec:experiments}

We conduct comprehensive experiments to examine the effectiveness of our method. We describe the experimental setup in \cref{sec:exp-setup}, present the main results in \cref{sec:exp-main}, and provide
analyses and ablations in \cref{sec:exp-ablations}.

\begin{table}[H]

\caption{Benchmarks used in our experiments.}
\label{tab:datasets}
\centering
\small
\resizebox{0.95\textwidth}{!}{

\begin{tabular}{l l l c}
\toprule
Benchmark & Task type & Knowledge range & Pool size \\
\midrule
MMLU-Redux \citep{gema2025we}    & Language & 57 academic subjects (cleaned MMLU)        & $3{,}000$  \\
MMLU-Pro \citep{wang2024mmlu}        & Language & STEM-heavy multi-domain reasoning          & $12{,}032$ \\
GPQA-Diamond \citep{rein2023gpqa}      & Language & Graduate-level science reasoning  & $198$      \\
MMBench \citep{liu2024mmbench}          & Vision-language & Multimodal perception and reasoning   & $4{,}329$  \\
\bottomrule
\end{tabular}}
\end{table}

\subsection{Experimental setup} \label{sec:exp-setup}

\paragraph{Datasets.} We evaluate on multiple benchmarks spanning advanced reasoning and perception tasks for modern large language and multimodal models. MMLU-Redux \citep{gema2025we} and MMLU-Pro \citep{wang2024mmlu} assess broad academic knowledge and chain-of-thought reasoning, GPQA-Diamond \citep{rein2023gpqa} targets graduate-level ``Google-proof'' science questions, and MMBench \citep{liu2024mmbench} evaluates vision-language perception and reasoning. Since open-ended generation evaluation remains an open problem, we focus on benchmarks with verifiable ground-truth answers, ensuring that estimation error reflects the active testing pipeline rather than metric noise. \cref{tab:datasets} summarizes the task type, knowledge domain, and pool size of each benchmark.

\paragraph{Models.} We study both large language and multimodal settings using the open-source Qwen-3.5 series \citep{qwen2026qwen35}, Qwen-3-VL series \citep{bai2025qwen3}, and advanced API-based GPT-5.4 series \citep{openai2026gpt54}. These model families cover both controlled open-weight scaling and frontier API deployment settings. For each surrogate-target pair, the weaker or smaller model is used as the surrogate $g$, while the stronger or larger model serves as the target $f$. All models generate free-form text, which is subsequently mapped to canonical answers via a parser for evaluation; we do not restrict the target to first-token-only decoding.

\paragraph{Baselines.} As shown in \cref{fig:lure_mse}, both PE-LURE and SE-LURE fail to outperform \emph{Uniform Sampling} across all evaluated benchmarks. We therefore exclude them from the main comparison and focus on baselines that can plausibly improve over Uniform Sampling. We use the MSE and relative MSE defined in \cref{sec:setting} as the evaluation metrics, and compare our method against three estimators: \emph{Uniform Sampling}, \emph{Proportional Allocation}, and \emph{Oracle-Neyman}.

{Uniform Sampling} selects $M$ samples uniformly without replacement for labeling. {Proportional Allocation} uses the same SE-based stratification as our method, but allocates the labeling budget proportionally to stratum size, i.e., $m_h \propto N_h$. This baseline serves as a within-framework reference: the gap between \emph{Uniform Sampling} and \emph{Proportional Allocation} isolates the effect of semantic-entropy-based stratification, while the additional gap between \emph{Proportional Allocation} and \emph{our method} captures the benefit of our proxy-Neyman allocation in \cref{eq:allocation}. \emph{Oracle-Neyman} allocates the budget according to the oracle target-side within-stratum standard deviation, i.e., $m_h \propto N_h \sigma_h$. As it relies on inaccessible target statistics, it is infeasible in practice and serves only as a theoretical upper bound on the best achievable allocation under the given stratification.

\begin{figure}[!t]
\centering
\includegraphics[width=0.99\linewidth]{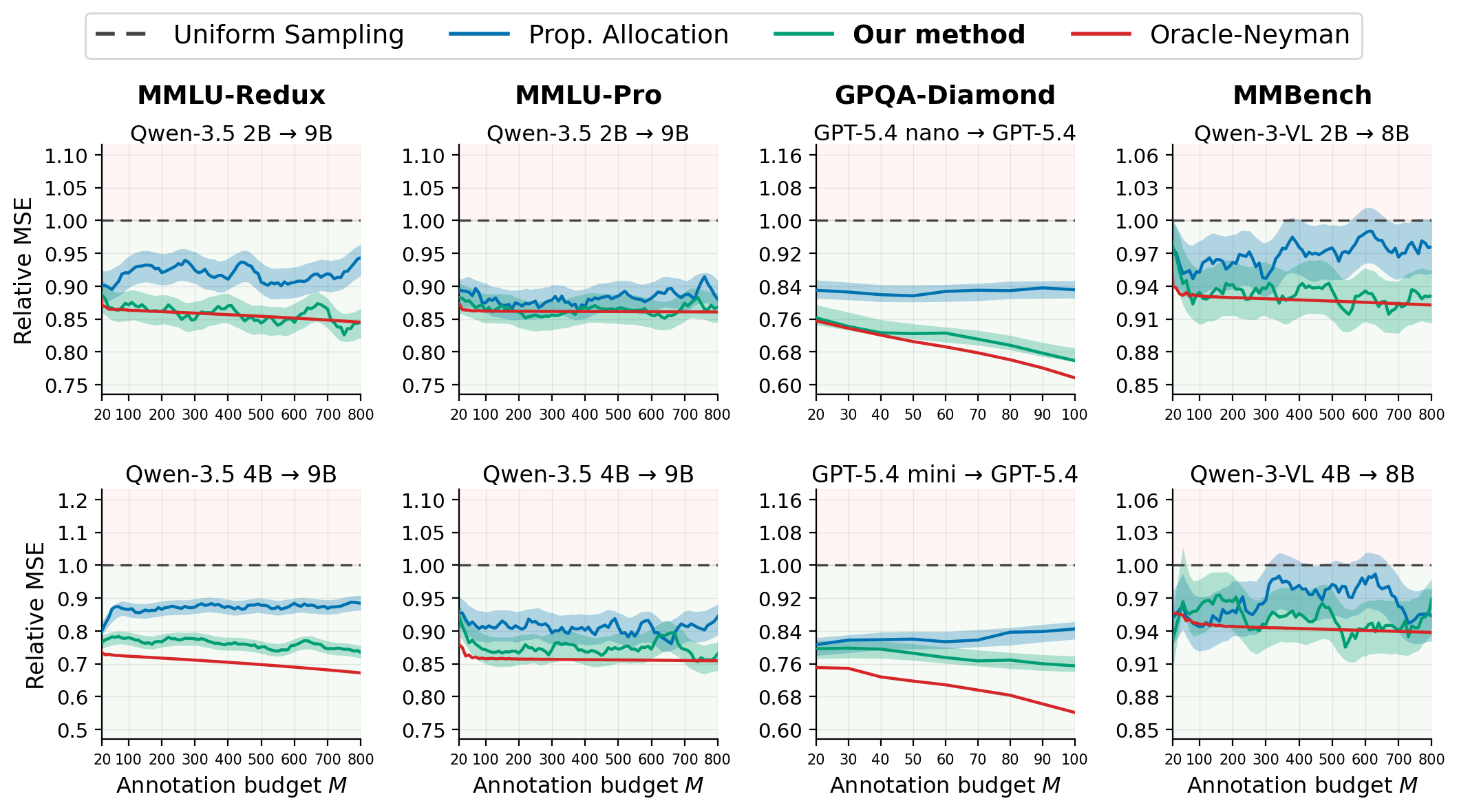}

\caption{Results of active testing across four benchmarks with various surrogate-target model pairs. We report relative MSE (\emph{lower is better}) as annotation budget $M$ increases. Our method consistently achieves better performance while closely matching the Oracle-Neyman (theoretical optimum). Each curve is averaged over $3{,}000$ Monte Carlo trials, with shaded bands indicating $\pm 1$ standard error of the mean.}

\label{fig:grid_2x4}
\end{figure}
\FloatBarrier

\subsection{Main results} \label{sec:exp-main}

\paragraph{Our method achieves consistent performance gains.}
We compare our method with all baselines across extensive benchmarks covering multiple domains and target model variants. As shown in
\cref{fig:grid_2x4}, our method consistently reduces estimation error over Uniform Sampling across all evaluated benchmarks, surrogate-target pairs, and labeling budgets. At the representative $M=70$ setting in \cref{fig:intro_budget70}, our method reduces MSE by up to $28\%$ relative to {Uniform Sampling}. The gains are most pronounced in the small-budget regime, where active evaluation is practically most useful and each additional label has the highest marginal cost. Across both language and vision-language benchmarks, our method remains below \emph{Uniform Sampling} throughout the budget sweep and closely follows the infeasible \emph{Oracle-Neyman} optimum, indicating that the proposed training-free pipeline captures much of the attainable variance reduction without target-side full-pool labels.

\begin{table}[!htbp]
\caption{Label savings across all benchmarks with different model pairs. We report the labeling budget required by Uniform Sampling and by our method to reach the same MSE, along with the resulting label savings. }

\label{tab:budget_savings}
\centering
\small
\begin{tabular}{l l c c c}
\toprule
Benchmark & Surrogate \,$\to$\, Target & Uniform Sampling & Our method & \textbf{Our savings} \\
\midrule
MMLU-Redux    & Qwen-3.5 4B $\to$ 9B        & $200$ & $144$ & \textbf{$28.0\%$} \\
MMLU-Pro      & Qwen-3.5 2B $\to$ 4B        & $340$ & $266$ & \textbf{$21.8\%$}          \\
GPQA-Diamond  & GPT-5.4 Nano $\to$ GPT-5.4  & $30$  & $22$  & \textbf{$26.7\%$}          \\
MMBench       & Qwen-3-VL 2B $\to$ 4B       & $80$  & $68$  & \textbf{$15.0\%$}          \\
\midrule
Average &  &  &  & \textbf{$22.9\%$} \\
\bottomrule
\end{tabular}
\end{table}

\paragraph{Our method yields noticeable label savings.} To quantify the label savings achieved by our method, we report the savings at matched MSE in \cref{tab:budget_savings}. We reframe the precision gains in terms of expert labeling cost, the key bottleneck in evaluation practice. As shown in \cref{tab:budget_savings}, our method saves $15.0\%$--$28.0\%$ of expert annotations across all four benchmarks. These savings are especially meaningful on reasoning benchmarks like GPQA-Diamond, where each science question typically requires co-verification by multiple domain experts.
\FloatBarrier

\subsection{Analyses and Ablations} \label{sec:exp-ablations}

We ablate the two central design choices of our method in the main text: the stratification method that maps surrogate semantic entropy onto bin boundaries, and the allocation rule that distributes the budget across strata. Two additional ablations, on the number of strata $H$ and on the offset $\delta$ in the proxy-Neyman allocation, are deferred to \cref{app:more-ablations}; both confirm that our default configuration ($H=5$, $\delta=0.75$) is at or near the optimum across budgets. All ablations are conducted on MMLU-Redux with the Qwen-3.5 2B surrogate against the Qwen-3.5 9B target, and we report \emph{relative MSE} (\%) versus \emph{Uniform Sampling} so that lower is better and $100\%$ corresponds to no improvement over \emph{Uniform Sampling}.

\paragraph{Selection of stratification method.}
To evaluate the effectiveness of our stratification strategy, we conduct additional experiments by replacing it with other stratification methods: \emph{Equal-width}, \emph{$k$-means}, and \emph{Quantile}, where {Equal-width} bins partition the SE range into equal-width intervals; \emph{$k$-means} performs data-adaptive 1-D clustering on SE values; \emph{Quantile} binning produces equal-frequency strata. We observe two main findings in \cref{tab:abl_stratmethod}. First, Equal-width binning performs the worst (86.1\% on average), as it over-concentrates samples near zero SE and under-represents high-SE regions. Second, data-adaptive methods such as $k$-means and Quantile binning perform similarly on average (84.8\% and 84.7\%, respectively), while Adaptive SE (Ours) achieves the best average performance (83.7\%), improving over them by isolating the zero-SE cluster where within-stratum variance is negligible. Based on these results, we adopt Adaptive SE stratification as the default strategy.

\begin{table}[!htbp]
\caption{Ablation study of \cref{alg:method} with different stratification methods (when $H$=5). We conduct experiments on the MMLU-Redux dataset, using Qwen-3.5 9B as the target model and Qwen-3.5 2B as the surrogate model. We report relative MSE (\%, \emph{lower is better}) versus Uniform Sampling and the best performance in each row is highlighted in bold.}

\label{tab:abl_stratmethod}
\centering
\small
\begin{tabular}{c c c c c}
\toprule
Budget $M$ & Equal-width & $k$-means & Quantile & Adaptive SE (Ours) $\downarrow$ \\
\midrule
$50$    & $89.9$ & $87.0$ & $86.1$ & $\mathbf{85.7}$ \\
$100$   & $89.3$ & $82.0$ & $87.7$ & $\mathbf{80.7}$ \\
$200$   & $84.1$ & $84.8$ & $\mathbf{81.8}$ & $82.5$ \\
$400$   & $\mathbf{84.5}$ & $84.6$ & $84.7$ & $86.9$ \\
$800$   & $82.7$ & $85.6$ & $83.3$ & $\mathbf{82.5}$ \\
\midrule
Average & $86.1$ & $84.8$ & $84.7$ & $\mathbf{83.7}$ \\
\bottomrule
\end{tabular}
\end{table}

\paragraph{Effectiveness of allocation strategy.}
We assess the effectiveness of our allocation strategy on MMLU-Redux using the Qwen-3.5 2B surrogate against the Qwen-3.5 9B target. As shown in \cref{tab:abl_alloc}, we fix SE-based stratification and compare five allocation strategies: 
\emph{Equal Allocation} ($m_h = M/H$), \emph{Proportional Allocation} ($m_h \propto N_h$), \emph{Power Allocation} ($m_h \propto \sqrt{N_h}$), \emph{our method} from \cref{eq:allocation}, and \emph{Oracle-Neyman} ($m_h \propto N_h \sigma_h$) which serves as an infeasible upper bound under the fixed stratification. Two findings emerge here. First, size-only heuristics are weak: Equal (90.9\%) and Proportional Allocation (91.4\%) retain about 90\% of the Uniform Sampling MSE, showing that stratification without variance information yields limited gains. Second, proxy-Neyman performs best at nearly all budgets and closely matches Oracle-Neyman on average (83.7\% vs.\ 82.8\%), while Power lies in between (86.3\%), indicating that the surrogate captures variance information beyond size alone.

\begin{table}[!htbp]
\caption{Ablation study of our method varying the allocation strategy on MMLU-Redux dataset with Qwen-3.5 2B (surrogate) and Qwen-3.5 9B (target) model pairs. We report the relative MSE (\%, \emph{lower is better}) versus Uniform Sampling as the annotation budget increases. \emph{Oracle-Neyman} (gray column) is infeasible in practice and serves as the theoretical optimum given the stratification; it is excluded from the per-row best-value comparison. The best value among feasible rules in each row is shown in bold.}

\label{tab:abl_alloc}
\centering
\small
\begin{tabular}{c c c c c c}
\toprule
Budget $M$ & Equal & Proportional & Power & Proxy-Neyman (Ours) $\downarrow$ & \textcolor{gray}{Oracle-Neyman} \\
\midrule
$50$    & $95.0$ & $92.6$ & $89.3$ & $\mathbf{85.7}$ & \textcolor{gray}{$84.3$} \\
$100$   & $95.1$ & $91.9$ & $85.4$ & $\mathbf{80.7}$ & \textcolor{gray}{$85.7$} \\
$200$   & $86.9$ & $86.6$ & $84.9$ & $\mathbf{82.5}$ & \textcolor{gray}{$82.0$} \\
$400$   & $90.0$ & $92.3$ & $\mathbf{86.2}$ & $86.9$ & \textcolor{gray}{$81.2$} \\
$800$   & $87.4$ & $93.4$ & $85.6$ & $\mathbf{82.5}$ & \textcolor{gray}{$80.7$} \\
\midrule
Average & $90.9$ & $91.4$ & $86.3$ & $\mathbf{83.7}$ & \textcolor{gray}{$82.8$} \\
\bottomrule
\end{tabular}
\end{table}

\section{Conclusion and Limitations} \label{sec:discussion}
We studied active testing for advanced generative language and multimodal models, where each label usually requires costly expert annotation and classification-oriented LURE estimators fail to outperform Uniform Sampling once the target model is evaluated through multi-token generation. We introduced a training-free and theoretically unbiased active-testing method that lifts the difficulty signal from tokens to meaning via semantic entropy, partitions the evaluation pool according to this signal, and allocates the label budget using a surrogate-driven proxy for Neyman allocation. Across language and vision-language reasoning benchmarks with both open-source and frontier proprietary surrogate--target pairs, our method consistently reduces estimation error relative to Uniform Sampling, closely tracks the infeasible Oracle-Neyman reference, and reaches matched precision with substantially fewer expert labels.

Our study also has some scope limitations. First, our experiments focus on closed or parseable answer spaces where generations can be mapped to canonical answers; extending active testing to open-ended generation remains challenging because evaluation noise can obscure estimator error. Second, the method assumes access to a surrogate model that is substantially cheaper than the target while still correlated with target-side difficulty, which may constrain deployment in some API-only scenarios. Future work can extend this framework by developing modality-aware surrogate signals and allocation rules for richer vision--language, interactive, and open-ended evaluation tasks.

\newpage
\bibliography{refs}

\newpage
\appendix
\section{Proof of the unbiasedness guarantee for our method} \label{app:proof}

\Unbiasedness*

\begin{proof}
Write $\ell_i = \ell(f(x_i), y_i)$ for brevity.
Fix a stratification $\{\cD_h\}_{h=0}^{H-1}$ of $\cD$ and an allocation $\{m_h\}_{h=0}^{H-1}$ with $1 \leq m_h \leq N_h$ and $\sum_{h=0}^{H-1} m_h = M$. Within each stratum $\cD_h$, let $\cS_h$ be a uniform sample without replacement of size $m_h$ drawn from $\cD_h$, and let $\bar\ell_h = \tfrac{1}{m_h}\sum_{i \in \cS_h} \ell_i$.
Since $\cS_h$ is sampled uniformly from all size-$m_h$ subsets of $\cD_h$, each such subset has probability $1/\binom{N_h}{m_h}$. For any fixed example $i \in \cD_h$, exactly $\binom{N_h-1}{m_h-1}$ of these subsets contain $i$, so
\[
\Pr\brk*{i \in \cS_h} \;=\; \frac{\binom{N_h-1}{m_h-1}}{\binom{N_h}{m_h}} \;=\; \frac{m_h}{N_h}, \qquad \text{for every } i \in \cD_h.
\]

Rewriting the sample mean as a sum over the full stratum, $m_h \bar\ell_h = \sum_{i \in \cD_h} \indic[i \in \cS_h]\,\ell_i$, and applying linearity of expectation yields
\[
\E[\bar\ell_h] \;=\; \frac{1}{m_h}\sum_{i \in \cD_h} \E\brk*{\indic[i \in \cS_h]}\ell_i \;=\; \frac{1}{m_h}\sum_{i \in \cD_h} \frac{m_h}{N_h}\,\ell_i \;=\; \frac{1}{N_h}\sum_{i \in \cD_h} \ell_i \;=\; R_h,
\]
the within-stratum risk. Because $\{\cD_h\}_{h=0}^{H-1}$ partitions $\cD$ and $\sum_{h=0}^{H-1} N_h = N$, a second application of linearity gives
\[
\E[\wh{R}] \;=\; \E\brk*{\frac{1}{N}\sum_{h=0}^{H-1} N_h \bar\ell_h}
            \;=\; \frac{1}{N}\sum_{h=0}^{H-1} N_h R_h
            \;=\; \frac{1}{N}\sum_{h=0}^{H-1} \sum_{i \in \cD_h} \ell_i
            \;=\; \frac{1}{N}\sum_{i=1}^N \ell_i
            \;=\; R_{\mathrm{D}},
\]

which establishes the unbiasedness claim. This is the usual Horvitz--Thompson inclusion-probability argument specialized to uniform sampling within each stratum.
\end{proof}

\section{Additional ablations} \label{app:more-ablations}

We present two additional ablations that are deferred from the main text to keep the experiments section focused. As in \cref{sec:exp-ablations}, all runs use MMLU-Redux with the Qwen-3.5 2B surrogate against the Qwen-3.5 9B target, and we report \emph{relative MSE} (\%) versus {Uniform Sampling} (lower is better).

\paragraph{Effect of $H$.}
To study the impact of the number of SE-based strata $H$, we vary $H$ from $2$ to $8$; results are shown in \cref{tab:abl_strata}. We observe a non-monotonic trend in relative MSE: smaller values of $H$ (e.g., $H=2$) fail to adequately separate easy and hard instances, leading to high within-stratum variance, while overly large values (e.g., $H=8$) result in insufficient samples per stratum for reliable estimation. The best performance is achieved at $H=5$, which attains the lowest average MSE (83.7\% relative to Uniform Sampling). Based on this result, we use $H=5$ as the default setting in all main experiments.

\begin{table}[!htbp]
\caption{\textbf{Ablation on the number of strata $H$.} Relative MSE (\%, lower is better) versus {Uniform Sampling} on MMLU-Redux, Qwen-3.5 2B surrogate $\to$ Qwen-3.5 9B target. Best value per row in bold.}
\label{tab:abl_strata}
\centering
\small
\begin{tabular}{c c c c c}
\toprule
Budget $M$ & $H = 2$ & $H = 3$ & $H = 5$ (Ours)$\downarrow$ & $H = 8$ \\
\midrule
$50$    & $90.3$ & $87.6$ & $\mathbf{85.7}$ & $89.3$ \\
$100$   & $91.3$ & $86.3$ & $\mathbf{80.7}$ & $83.9$ \\
$200$   & $86.5$ & $85.5$ & $\mathbf{82.5}$ & $83.4$ \\
$400$   & $91.7$ & $\mathbf{85.4}$ & $86.9$ & $86.9$ \\
$800$   & $82.6$ & $84.2$ & $\mathbf{82.5}$ & $83.7$ \\
\midrule
Average ($\downarrow$) & $88.5$ & $85.8$ & $\mathbf{83.7}$ & $85.5$ \\
\bottomrule
\end{tabular}
\end{table}

\paragraph{Impact of offset $\delta$.}
To evaluate the impact of the offset $\delta$ in the proxy-Neyman allocation $m_h \propto N_h\paren*{\sqrt{p_h(1-p_h)} + \delta}$, we sweep $\delta \in \{0.5, 0.75, 1, 2, 5\}$. As shown in \cref{tab:abl_offset}, the sweep reveals a broad optimum: $\delta \in [0.5, 1]$ achieves similar performance, with average relative MSE between $83.7\%$ and $84.0\%$, while performance degrades to $86.2\%$ at $\delta=2$ and $87.6\%$ at $\delta=5$. This trend reflects a trade-off in the surrogate-side proxy: small $\delta$ amplifies noise in $\sqrt{p_h(1-p_h)}$ by approaching pure proxy-Neyman allocation, while large $\delta$ flattens the allocation toward Proportional Allocation and weakens the variance signal. We adopt $\delta=0.75$ as the default, which achieves the lowest average relative MSE in the sweep.

\begin{table}[!htbp]
\caption{\textbf{Ablation on the offset $\delta$ in the proxy-Neyman allocation.} Relative MSE (\%, lower is better) versus {Uniform Sampling} on MMLU-Redux, Qwen-3.5 2B surrogate $\to$ Qwen-3.5 9B target. Best value per row in bold.}
\label{tab:abl_offset}
\centering
\small
\begin{tabular}{c c c c c c}
\toprule
Budget $M$ & $\delta = 0.5$ & $\delta = 0.75$ (Ours)$\downarrow$ & $\delta = 1$ & $\delta = 2$ & $\delta = 5$ \\
\midrule
$50$    & $86.0$ & $\mathbf{85.7}$ & $85.8$ & $89.0$ & $89.8$ \\
$100$   & $82.2$ & $80.7$ & $\mathbf{80.6}$ & $82.3$ & $83.5$ \\
$200$   & $83.8$ & $\mathbf{82.5}$ & $82.9$ & $85.3$ & $87.1$ \\
$400$   & $\mathbf{86.3}$ & $86.9$ & $87.1$ & $89.7$ & $90.2$ \\
$800$   & $\mathbf{81.9}$ & $82.5$ & $83.2$ & $84.6$ & $87.4$ \\
\midrule
Average & $84.0$ & $\mathbf{83.7}$ & $83.9$ & $86.2$ & $87.6$ \\
\bottomrule
\end{tabular}
\end{table}
\begin{table}[htbp]
\caption{\textbf{Hyperparameters used in our experiments.} Top block: decoding configuration for model inference. Bottom block: active-testing hyperparameters of our method.}
\label{tab:hyperparams}
\centering
\small
\begin{tabular}{l c l}
\toprule
\multicolumn{3}{c}{\textbf{Model inference}} \\
\midrule
Parameter & Value & Description \\
\midrule
$k$                          & $10$        & Number of generations per input \\
\texttt{temperature}         & $0.7$       & Sampling temperature \\
\texttt{top\_p}              & $0.8$       & Nucleus sampling threshold \\
\texttt{top\_k}              & $20$        & Top-$k$ sampling cutoff \\
\texttt{presence\_penalty}   & $1.5$       & Presence penalty \\
\texttt{repetition\_penalty} & $1.0$       & Repetition penalty \\
\texttt{max\_new\_tokens}    & model-max   & Maximum supported by each model \\
\midrule[\heavyrulewidth]
\multicolumn{3}{c}{\textbf{Active testing (our method)}} \\
\midrule
Parameter & Value & Description \\
\midrule
$H$        & $5$       & Number of semantic-entropy strata \\
$\delta$   & $0.75$    & Proxy-Neyman allocation offset (\cref{eq:allocation}) \\
$T$        & $3{,}000$ & Monte Carlo trials per labeling budget \\
\bottomrule
\end{tabular}
\end{table}
\section{Reproducibility and compute} \label{app:compute}

\paragraph{Models, data, and code.}
All datasets are downloaded from Hugging Face using the public benchmark releases listed in \cref{tab:datasets}. For open-weight experiments, we download the official Qwen family from Hugging Face. For API-based experiments, we use the official OpenAI API service. We provide full algorithmic details, hyperparameters, and generation/parsing protocols in this appendix to support reproduction without released code.

\paragraph{Generation and parsing details.}\label{app:parsing}
Each GPU is used independently for inference with the vLLM framework on NVIDIA RTX 6000 Ada hardware. Surrogate signals are generated with the vLLM decoding configuration listed in \cref{tab:hyperparams}. We parse model generations with the official standard parser script released by each benchmark. This benchmark-official parsing protocol may differ from model-official self-reporting pipelines, which may add an LLM judge as a final fallback after rule-based parsing. Parsing failures account for less than $1\%$ in every pool in our experiments and have minimal impact on the reported experimental precision.

\paragraph{Active-testing hyperparameters.}
Unless otherwise stated, our method uses the hyperparameters summarized in \cref{tab:hyperparams}. We describe the parameter-selection process in \cref{sec:exp-ablations}.

\end{document}